\documentclass{article}

\usepackage{microtype}
\usepackage{graphicx}
\usepackage{subcaption}
\usepackage{booktabs}
\usepackage{hyperref}

\usepackage[accepted]{icml2026}
\setcitestyle{numbers,square,sort&compress,comma}

\usepackage{amsmath}
\usepackage{amssymb}
\usepackage{mathtools}
\usepackage{amsthm}
\usepackage[capitalize,noabbrev]{cleveref}

\theoremstyle{plain}
\newtheorem{theorem}{Theorem}          
\newtheorem{proposition}{Proposition}  
\newtheorem{corollary}{Corollary}      

\theoremstyle{definition}
\newtheorem{assumption}{Assumption}    

\begin{document}
\onecolumn

\icmltitle{Robustness of Agentic AI Systems via \\ Adversarially-Aligned Jacobian Regularization}

\vskip 0.15in

\begin{center}
{\fontsize{11.25pt}{13.5pt}\bfseries\selectfont
Furkan Mumcu,
Yasin Yilmaz
}

{\fontsize{11.25pt}{13.5pt}\selectfont
University of South Florida
}

{\tt\small
\{furkan, yasiny\}@usf.edu
}
\end{center}

\vskip 0.25in

\begin{abstract} 
As Large Language Models (LLMs) transition into autonomous multi-agent ecosystems, robust minimax training becomes essential yet remains prone to instability when highly non-linear policies induce extreme local curvature in the inner maximization. Standard remedies that enforce global Jacobian bounds are overly conservative, suppressing sensitivity in all directions and inducing a large \emph{Price of Robustness}. We introduce \emph{Adversarially-Aligned Jacobian Regularization} (AAJR), a trajectory-aligned approach that controls sensitivity strictly along adversarial ascent directions. We prove that AAJR yields a strictly larger admissible policy class than global constraints under mild conditions, implying a weakly smaller approximation gap and reduced nominal performance degradation. Furthermore, we derive step-size conditions under which AAJR controls effective smoothness along optimization trajectories and ensures inner-loop stability. These results provide a structural theory for agentic robustness that decouples minimax stability from global expressivity restrictions.
\end{abstract}
\section{Introduction}
\label{sec:introduction}

The deployment of Large Language Models (LLMs) is rapidly shifting from isolated, single-turn interactions to autonomous, multi-agent ecosystems. In these environments, agents must not only optimize for localized tasks but also maintain stability against adversarial shifts, competing objectives, and system-level congestion. To guarantee worst-case performance and system robustness, training these agents can naturally be formulated as a minimax optimization problem \cite{madry2017towards, xhonneux2024efficient}. However, while minimax frameworks provide rigorous theoretical guarantees in convex settings, scaling them to highly non-linear systems like deep neural networks introduces profound optimization challenges.

A primary bottleneck in training robust agentic systems via gradient-based minimax learners is the instability of Gradient Descent-Ascent (GDA). Because deep neural networks are highly expressive, unconstrained internal propagation dynamics can allow the inner maximization loop (the adversary) to encounter regions of extreme local curvature, leading to limit cycles or divergence. To secure convergence, traditional stabilization methods such as spectral normalization or standard adversarial training, implicitly or explicitly enforce a global bound on the network's local Lipschitz constant, typically by limiting the spectral norm of the \emph{state-action Jacobian} across the state domain.

While such global constraints can stabilize the minimax learner, they introduce a trade-off commonly described as the \emph{Price of Robustness} \cite{bertsimas2004price}. By globally choking sensitivity, the model's admissible hypothesis class is severely restricted. This strictly increases the approximation gap relative to the unconstrained policy class, embedding a structural Price of Robustness in exchange for worst-case stability. Existing literature often treats this trade-off as an unavoidable consequence of adversarial robustness \cite{tsipras2018robustness, zhang2019theoretically}.

In this work, we challenge the necessity of this strict trade-off. We argue that global Jacobian constraints are inherently overly pessimistic: the inner maximization loop does not exploit the entire state space uniformly, but instead follows localized ascent trajectories that expose directions of maximum vulnerability. Consequently, suppressing sensitivity in all orthogonal, task-relevant directions can be mathematically unnecessary for GDA stability and destructive to nominal expressivity.

This limitation becomes particularly pronounced in autonomous agentic systems operating in interactive environments. Global sensitivity constraints have historically provided reasonable robustness–utility trade-offs for standard predictive models. Passive predictors can tolerate globally restricted input sensitivity because their outputs do not directly influence future states. In contrast, agentic systems must continuously adapt their actions within dynamic, shared environments. Enforcing uniform global bounds on state sensitivity therefore restricts responsiveness across all directions of the state space, effectively shrinking the admissible policy class. As a result, the resulting \emph{Price of Robustness} becomes especially severe for agentic systems, where expressive and context-dependent behavior is essential for coordination and planning.

To bridge internal propagation dynamics and minimax stability, we introduce \emph{Adversarially-Aligned Jacobian Regularization} (AAJR). Rather than enforcing a rigid global Lipschitz bound, AAJR adaptively suppresses sensitivity strictly along adversarial ascent directions induced by the inner maximization procedure. This directional perspective decouples inner-loop stability from global expressivity restrictions. Our theory establishes (i) an expressivity gain via strict hypothesis class expansion under mild conditions, and (ii) stability of the inner maximization dynamics via trajectory-wise curvature control and explicit step-size conditions.

The main contributions of this paper are as follows:
\begin{itemize}
    \item \textbf{Bottleneck Formalization for Agentic Minimax Learning:}
    We formalize the expressivity versus stability tension in gradient-based minimax training for agentic systems by showing that global Jacobian control restricts the admissible policy class and induces a Price of Robustness in nominal risk.

    \item \textbf{Trajectory-Aligned Sensitivity Control:}
    We propose \emph{Adversarially-Aligned Jacobian Regularization} (AAJR), which suppresses state sensitivity only along adversarial ascent directions generated by the inner maximization, rather than enforcing a global Jacobian bound.

    \item \textbf{Expressivity Guarantee via Class Expansion:}
    We introduce a trajectory-adaptive hypothesis class induced by directional constraints and prove that the globally constrained class is strictly contained in it, implying a weakly smaller approximation gap and thus a reduced Price of Robustness relative to global sensitivity control.

    \item \textbf{Optimization Guarantees for Inner Maximization Stability:}
    Under our standing smoothness assumptions, we show that bounding adversarially-aligned Jacobian amplification yields an explicit bound on the \emph{effective smoothness} of the inner objective along projected gradient ascent iterates, and we derive step-size conditions ensuring stable inner-loop dynamics that avoid curvature-driven divergence.
\end{itemize}

\section{Preliminaries}
\label{sec:preliminaries}

In this section, we establish the notation and foundational concepts regarding empirical risk minimization, minimax optimization, and the mathematical properties governing sensitivity of non-linear systems.

\subsection{Notation and Multi-Agent Objective}
\label{subsec:notation}
For agent $i$, let $\mathcal{D}$ be a distribution over an environment state space $\mathcal{S} \subset \mathbb{R}^d$ and a peer action/context space $\mathcal{A}_{-i}$. We consider an autonomous agent parameterized by a deep neural network policy $\pi_\theta: \mathcal{S} \rightarrow \mathcal{A}$ with weights $\theta \in \Theta$, where $\mathcal{A}$ is the agent's action space. The standard goal of learning in this context is to find nominal parameters $\theta_{nom}^*$ that minimize the expected system-level loss (e.g., negative aggregate utility) $\mathcal{L}$:
\begin{equation}
\theta_{nom}^* = \arg\min_{\theta \in \Theta} \; \mathbb{E}_{(s,a_{-i}) \sim \mathcal{D}} \big[\mathcal{L}(\pi_\theta(s), a_{-i})\big].
\end{equation}
When $\mathcal{D}$ is unknown, the expectation is typically approximated by an empirical average over a finite sample of environment states and peer interactions.

\subsection{Minimax Optimization and GDA Dynamics}
\label{subsec:minimax}
In multi-agent environments requiring robustness against worst-case environmental shifts or adversarial perturbations $\delta \in \Delta$ (where $\Delta = \{\delta \in \mathbb{R}^d : \|\delta\|_p \le \epsilon\}$), the standard nominal objective is insufficient. A common robust objective is the following minimax problem:
\begin{equation}
\min_{\theta \in \Theta} \; \mathbb{E}_{(s,a_{-i}) \sim \mathcal{D}}\Big[ \max_{\delta \in \Delta} \mathcal{L}\big(\pi_\theta(s + \delta), a_{-i}\big) \Big].
\end{equation}
In this agentic setting, $s$ represents the agent's observation of the environment state (e.g., current resource availability) and $a_{-i}$ encapsulates the multi-agent context, including the observable actions or demands of peer agents. Crucially, an adversarial perturbation $\delta$ mathematically represents a systemic shock to the shared state (e.g., a sudden, destabilizing spike in peer demand) rather than a standard text-level adversarial prompt. Correspondingly, $\mathcal{L}$ is a system-level objective induced by the agent's behavior, such as negative aggregate utility $-\mathcal{U}_{sys}(\pi_\theta(s), a_{-i})$ or system instability. Thus, the inner maximization corresponds to worst-case state shifts that degrade overall system outcomes.

This objective is typically optimized using Gradient Descent-Ascent (GDA) or its variants (e.g., Projected Gradient Descent for the inner loop) \cite{nouiehed2019solving, sinha2018certifying, xian2021faster}. The inner maximization step plays the role of the adversary or worst-case environment, taking projected gradient ascent steps with respect to $\delta$ (equivalently, with respect to the perturbed state $s+\delta$). The outer minimization step plays the role of the agent, taking gradient descent steps with respect to the policy parameters $\theta$.

\subsection{Lipschitz Continuity and the Policy Jacobian}
\label{subsec:lipschitz}
For a non-linear system such as a deep neural network policy, the behavior of the inner maximization loop is governed by how sensitively the agent's actions respond to small state variations. A policy $\pi_\theta$ is $L$-Lipschitz continuous with respect to the $\ell_2$-norm if there exists a constant $L \ge 0$ such that for all $s_1, s_2 \in \mathcal{S}$:
\begin{equation}
\|\pi_\theta(s_1) - \pi_\theta(s_2)\|_2 \le L \|s_1 - s_2\|_2.
\end{equation}

If $\pi_\theta$ is differentiable, then for any convex set $\mathcal{S}$ one has the standard bound
\begin{equation}
\sup_{s \in \mathcal{S}} \|J(\theta,s)\|_2 \;\le\; L,
\end{equation}
where $J(\theta,s)=\nabla_s \pi_\theta(s)$ is the state-action Jacobian. In particular, bounding $\sup_{s \in \mathcal{S}} \|J(\theta,s)\|_2$ provides a sufficient condition for $\ell_2$-Lipschitz continuity on $\mathcal{S}$, which mathematically limits the extent to which small environmental shocks cascade into massive shifts in the agent's behavior.
\section{Problem Formulation}
\label{sec:problem_formulation}

We formalize the central bottleneck in robust learning for agentic systems trained via gradient-based min-max procedures: ensuring stability of the inner maximization dynamics without globally restricting the policy class and incurring a large nominal performance penalty.

\subsection{Stability Requirement and Global Sensitivity Control}
\label{subsec:stability}
We study robust training objectives of the form
\begin{equation}
\min_{\theta \in \Theta}\; \mathbb{E}_{(s,a_{-i})\sim\mathcal{D}}
\Big[ \max_{\delta \in \Delta} \mathcal{L}\big(\pi_\theta(s+\delta), a_{-i}\big) \Big],
\end{equation}
typically optimized by Gradient Descent-Ascent (GDA) or Projected Gradient Ascent (PGA/PGD) in the inner loop. In highly non-linear policies, the inner maximization can exhibit oscillation or divergence when local curvature along the ascent trajectory is large, making the overall min-max training unstable. In agentic settings, such instability can manifest as large swings in an agent's actions under small state shifts, potentially cascading into system-level failures.

A common stabilization approach enforces a global bound on state sensitivity by constraining the operator norm of the state-action Jacobian:
\begin{equation}
\sup_{s \in \mathcal{S}} \| J_\theta(s) \|_2 \le \gamma,
\label{eq:global_jacobian_bound}
\end{equation}
where $J_\theta(s) \coloneqq \nabla_s \pi_\theta(s)$ and $\gamma>0$ is a global sensitivity budget. Such bounds are sufficient (though generally conservative) for limiting worst-case amplification of state perturbations and, in turn, stabilizing gradient-based inner maximization.

\subsection{Globally Constrained Policy Class}
\label{subsec:restricted_class}
Let $\mathcal{F}$ denote the unconstrained policy class. The global constraint~\eqref{eq:global_jacobian_bound} induces the restricted class
\begin{equation}
\mathcal{F}_\gamma
\;=\;
\Big\{
\pi \in \mathcal{F} : \sup_{s \in \mathcal{S}} \|J_\pi(s)\|_2 \le \gamma
\Big\}.
\end{equation}
As robustness requirements increase (e.g., larger perturbation radius $\epsilon$ in $\Delta$), practical procedures often require stronger sensitivity control, which strictly reduces the volume of the admissible class $\mathcal{F}_\gamma$ and mathematically forces a larger lower bound on the nominal approximation gap.

\subsection{Price of Robustness}
\label{subsec:price_of_robustness}
We quantify this degradation via the nominal risk
\begin{equation}
\mathcal{R}_{\mathrm{nom}}(\pi)
=
\mathbb{E}_{(s,a_{-i})\sim\mathcal{D}}
\big[\mathcal{L}(\pi(s),a_{-i})\big],
\end{equation}
and define the \emph{Price of Robustness} induced by global sensitivity control as the approximation gap
\begin{equation}
\mathcal{T}(\gamma)
=
\inf_{\pi \in \mathcal{F}_\gamma}\mathcal{R}_{\mathrm{nom}}(\pi)
-
\inf_{\pi \in \mathcal{F}}\mathcal{R}_{\mathrm{nom}}(\pi).
\end{equation}
Intuitively, smaller $\gamma$ yields a more restrictive class $\mathcal{F}_\gamma$ and tends to increase $\mathcal{T}(\gamma)$.

\paragraph{Core bottleneck.}
Global Jacobian control can stabilize the inner maximization but may be overly pessimistic, suppressing sensitivity in directions that are irrelevant to adversarial ascent and thereby increasing $\mathcal{T}(\gamma)$. This motivates seeking stability guarantees that depend only on \emph{trajectory-relevant} directions induced by the inner maximization dynamics.


\subsection{Standing Assumptions}  
\label{subsec:assumptions}
We state the conditions under which we develop our theoretical guarantees.

\begin{assumption}[Loss smoothness]\label{assump:loss_smooth} For each multi-agent context $a_{-i}$, the function $z \mapsto \mathcal{L}(z, a_{-i})$ is $L_{\mathcal{L}}$-smooth in the $\ell_2$ norm. That is, for all actions $z_1, z_2 \in \mathcal{A}$, the gradient $\nabla_z \mathcal{L}$ satisfies the Lipschitz condition:
\begin{equation}
    \|\nabla_z \mathcal{L}(z_1, a_{-i}) - \nabla_z \mathcal{L}(z_2, a_{-i})\|_2 \le L_{\mathcal{L}} \|z_1 - z_2\|_2
\end{equation}

\end{assumption}

\begin{assumption}[Differentiable policy]\label{assump:policy_diff}
The policy $\pi_\theta$ is differentiable in $s$ for all $\theta\in\Theta$ and $s\in\mathcal{S}$, with Jacobian $J_\theta(s)=\nabla_s \pi_\theta(s)$.
\end{assumption}

\begin{assumption}[Bounded second-order term along PGA iterates]\label{assump:second_order}Along the line segments between the inner-loop PGA iterates $\{\delta_t\}_{t=0}^{K}$, the policy's second-order contribution to the inner objective curvature is bounded by a constant $C\ge 0$. Specifically, letting $z = \pi_\theta(s+\delta)$, we assume the spectral norm of the residual Hessian term satisfies:$$ \left\| \sum_j \frac{\partial \mathcal{L}(z, a_{-i})}{\partial z_j} \nabla_\delta^2 \pi_{\theta,j}(s+\delta) \right\|_2 \le C $$for all $\delta \in [\delta_t, \delta_{t+1}]$ and $t = 0, \dots, K-1$.\end{assumption}

\begin{assumption}[Compact perturbation set and finite inner steps]\label{assump:compact}
The perturbation set $\Delta$ is compact and convex, and the inner maximization uses $K<\infty$ projected gradient steps with step size $\eta>0$.
\end{assumption}

These assumptions are standard in nonconvex minimax analysis \cite{lin2020gradient, madry2017towards, wang2021adversarial} and allow us to relate stability of the inner maximization to Jacobian-controlled propagation dynamics along adversarial trajectories.

\section{Directional Jacobian Constraints for Structural Agentic Robustness}
\label{sec:methodology}

Section~\ref{sec:problem_formulation} showed that stabilizing minimax learning via a global Jacobian bound can be overly conservative: it restricts the hypothesis class and can increase the Price of Robustness $\mathcal{T}(\gamma)$. We now introduce a \emph{trajectory-aligned} alternative that controls sensitivity only along adversarial ascent directions induced by the inner maximization dynamics.

\subsection{Adversarial Ascent Trajectories and Directional Sensitivity}
\label{subsec:directional_jacobian}

Fix a policy $\pi_\theta$ and a sample $(s,a_{-i})\sim \mathcal{D}$. Consider $K$ steps of projected gradient ascent in the perturbation space $\Delta$:

\begin{equation}
\delta_{0}=0,
\quad
\delta_{t+1} = \Pi_{\Delta}\!\Big(\delta_t + \eta \nabla_\delta \mathcal{L}(\pi_\theta(s+\delta_t), a_{-i})\Big),
\quad t=0,\ldots,K-1.
\label{eq:pga_update_clean}
\end{equation}
where $\eta>0$ is the step size and $\Pi_\Delta$ denotes projection onto $\Delta$. To evaluate the policy's sensitivity purely along this trajectory, we isolate the normalized ascent direction:
\begin{equation}
u_t
=
\frac{\nabla_{\delta}\,\mathcal{L}\!\left(\pi_\theta(s+\delta_t),a_{-i}\right)}
{\left\|\nabla_{\delta}\,\mathcal{L}\!\left(\pi_\theta(s+\delta_t),a_{-i}\right)\right\|_2
+\varepsilon_0},
\qquad \varepsilon_0>0.
\label{eq:ut_def_clean}
\end{equation}
By definition, $\|u_t\|_2 \le 1$. The resulting adversarial perturbation after $K$ steps is $\delta^\star \coloneqq \delta_K$. Our central object is the \emph{directional Jacobian amplification} along the ascent directions:
\begin{equation}
\left\|J_\theta(s+\delta_t)\,u_t\right\|_2. 
\label{eq:dir_amp}
\end{equation}
Instead of globally bounding $\|J_\theta(s)\|_2$, we impose a \emph{trajectory-aligned directional budget}:
\begin{equation}
\left\|J_\theta(s+\delta_t)\,u_t\right\|_2
\le
\gamma_{\mathrm{adv}},
\qquad t=0,\ldots,K-1,
\label{eq:dir_budget_traj_clean}
\end{equation}
where $\gamma_{\mathrm{adv}}>0$ controls amplification \emph{only} in directions that the inner maximization actually exploits.

\subsection{Adaptive Hypothesis Class and Price-of-Robustness Implications}
\label{subsec:expanded_class}

Let $\{\delta_t(s,a_{-i};\pi)\}_{t=0}^{K}$ be the deterministic PGA iterates generated by~\eqref{eq:pga_update_clean} applied to $\delta \mapsto \mathcal{L}(\pi(s+\delta),a_{-i})$, with directions $\{u_t(s,a_{-i};\pi)\}$ defined by~\eqref{eq:ut_def_clean}. We define the \emph{trajectory-adaptive} policy class as
\begin{equation}
\mathcal{F}_{\mathrm{ad}}(\gamma_{\mathrm{adv}})
=
\Big\{
\pi\in\mathcal{F} :
\left\|J_\pi(s+\delta_t)\,u_t(s,a_{-i};\pi)\right\|_2
\le
\gamma_{\mathrm{adv}},
\ \forall t=0,\ldots,K-1,\ 
\mathcal{D}\text{-a.e. }(s,a_{-i})
\Big\}.
\label{eq:F_adaptive_clean}
\end{equation}

Recall the globally constrained class
\begin{equation}
\mathcal{F}_\gamma
=
\Big\{\pi\in\mathcal{F} : \sup_{s\in\mathcal{S}} \|J_\pi(s)\|_2 \le \gamma \Big\}.
\label{eq:F_global_clean}
\end{equation}

\begin{proposition}[Global constraints imply directional constraints]
\label{prop:inclusion}
If $\pi\in \mathcal{F}_\gamma$ and $\gamma \le \gamma_{\mathrm{adv}}$, then $\pi \in \mathcal{F}_{\mathrm{ad}}(\gamma_{\mathrm{adv}})$.
\end{proposition}

\noindent
\textbf{Proof sketch.} For any unit vector $u$, $\|J_\pi(x)u\|_2 \le \|J_\pi(x)\|_2$. Since $\|u_t\|_2 \le 1$ by construction, the global bound implies~\eqref{eq:dir_budget_traj_clean} along the trajectory. \hfill $\square$

As a consequence, for $\gamma\le \gamma_{\mathrm{adv}}$ we have
\begin{equation}
\mathcal{F}_\gamma \subseteq \mathcal{F}_{\mathrm{ad}}(\gamma_{\mathrm{adv}}) \subseteq \mathcal{F}.
\label{eq:class_inclusion_clean}
\end{equation}
Define the corresponding approximation gaps:
\begin{equation}
\mathcal{T}(\gamma)
=
\inf_{\pi\in\mathcal{F}_\gamma}\mathcal{R}_{\mathrm{nom}}(\pi)
-
\inf_{\pi\in\mathcal{F}}\mathcal{R}_{\mathrm{nom}}(\pi),
\qquad
\mathcal{T}_{\mathrm{ad}}(\gamma_{\mathrm{adv}})
=
\inf_{\pi\in\mathcal{F}_{\mathrm{ad}}(\gamma_{\mathrm{adv}})}\mathcal{R}_{\mathrm{nom}}(\pi)
-
\inf_{\pi\in\mathcal{F}}\mathcal{R}_{\mathrm{nom}}(\pi).
\label{eq:approx_gaps_clean}
\end{equation}
The inclusion~\eqref{eq:class_inclusion_clean} implies
\begin{equation}
\mathcal{T}_{\mathrm{ad}}(\gamma_{\mathrm{adv}}) \le \mathcal{T}(\gamma).
\label{eq:gap_inequality_clean}
\end{equation}
In Section~\ref{sec:theory}, we provide conditions under which the inclusion $\mathcal{F}_\gamma \subsetneq \mathcal{F}_{\mathrm{ad}}(\gamma_{\mathrm{adv}})$ is strict, yielding a strictly improved robustness-expressivity trade-off.

\subsection{A Practical Surrogate: Adversarially-Aligned Jacobian Regularization}
\label{subsec:optimization_objective}

The constraint~\eqref{eq:dir_budget_traj_clean} motivates a tractable regularizer that penalizes directional Jacobian amplification along the inner-loop ascent trajectory. Given PGA iterates $\{\delta_t\}_{t=0}^{K}$ and ascent directions $\{u_t\}_{t=0}^{K-1}$, define
\begin{equation}
\mathcal{R}_{\mathrm{AAJR}}(\theta; s,a_{-i})
=
\frac{1}{K}\sum_{t=0}^{K-1}
\left\|J_\theta(s+\delta_t)\,\mathrm{stopgrad}(u_t)\right\|_2^2 .
\label{eq:aajr_traj_clean}
\end{equation}
We then consider the regularized robust objective
\begin{equation}
\min_{\theta\in\Theta}\;
\mathbb{E}_{(s,a_{-i})\sim\mathcal{D}}
\Big[
\max_{\delta\in\Delta}\;
\mathcal{L}\!\left(\pi_\theta(s+\delta),a_{-i}\right)
+\lambda\,\mathcal{R}_{\mathrm{AAJR}}(\theta; s,a_{-i})
\Big],
\label{eq:robust_obj_aajr_clean}
\end{equation}
where $\lambda\ge 0$ controls the strength of trajectory-aligned sensitivity suppression. The $\mathrm{stopgrad}(\cdot)$ operator indicates that gradients are not backpropagated through $\{u_t\}$, yielding a stable first-order surrogate analogous to standard adversarial training.

\paragraph{Agentic interpretation.}
In agentic systems, $\mathcal{L}$ captures system-level performance induced by an agent's actions in a multi-agent context (e.g., aggregate utility or stability). Although $\mathcal{R}_{\mathrm{AAJR}}$ acts on an individual agent's internal propagation through $J_\theta$, the directions $\{u_t\}$ are determined by ascent on the system loss $\mathcal{L}$. Thus AAJR suppresses sensitivity \emph{only} along directions that most degrade system outcomes under worst-case shifts, rather than globally across all state directions.
\section{Theoretical Guarantees}
\label{sec:theory}

We provide formal guarantees for trajectory-aligned directional Jacobian control. First, we show that replacing a global Jacobian constraint with a trajectory-aligned directional constraint expands the admissible policy class, yielding a (weakly) smaller approximation gap and thus a lower Price of Robustness. Second, under the standing assumptions of Section~\ref{subsec:assumptions}, we show that the directional constraint upper bounds the \emph{effective smoothness} of the inner maximization objective along the PGA iterates, which controls curvature growth and stabilizes the inner-loop dynamics.

\subsection{Expressivity: Class Expansion and Price-of-Robustness}
\label{subsec:class_expansion}

Recall the globally constrained class $\mathcal{F}_\gamma$ in~\eqref{eq:F_global_clean} and the trajectory-adaptive class $\mathcal{F}_{\mathrm{ad}}(\gamma_{\mathrm{adv}})$ in~\eqref{eq:F_adaptive_clean}.

\begin{theorem}[Class inclusion and strict expansion]
\label{thm:strict_inclusion_clean}
Fix $\gamma>0$ and choose $\gamma_{\mathrm{adv}}=\gamma$. Then
\[
\mathcal{F}_\gamma \subseteq \mathcal{F}_{\mathrm{ad}}(\gamma).
\]
Moreover, the inclusion is strict, $\mathcal{F}_\gamma \subsetneq \mathcal{F}_{\mathrm{ad}}(\gamma)$, whenever the ascent directions induced by the inner maximization do not span all directions on a set of positive $\mathcal{D}$-measure; concretely, if there exists a set $E$ with $\mathcal{D}(E)>0$ such that for all $(s,a_{-i})\in E$ and all $t\in\{0,\ldots,K-1\}$,
\[
u_t(s,a_{-i};\pi) \in U \subsetneq \mathbb{R}^d
\quad\text{for some fixed proper subspace }U,
\]
then there exists a policy $\pi\in \mathcal{F}_{\mathrm{ad}}(\gamma)$ with $\pi\notin \mathcal{F}_\gamma$.
\end{theorem}

\begin{proof}
\textbf{Inclusion.}
Let $\pi\in\mathcal{F}_\gamma$. Then for all $x\in\mathcal{S}$ and all unit vectors $v$,
$\|J_\pi(x)v\|_2 \le \|J_\pi(x)\|_2 \le \gamma$.
Since $\|u_t\|_2\le 1$ by construction, it follows that along any PGA trajectory,
$\|J_\pi(s+\delta_t)u_t\|_2 \le \gamma$ for all $t$, hence $\pi\in\mathcal{F}_{\mathrm{ad}}(\gamma)$.

\textbf{Strictness (construction).}
Assume the ascent directions lie in a fixed proper subspace $U\subsetneq \mathbb{R}^d$ on a set $E$ of positive measure. Let $P_U$ be the orthogonal projector onto $U$ and $P_{U^\perp}$ onto its orthogonal complement. Consider a (local) Jacobian of the form
\[
J = \gamma\,P_U \;+\; M\,P_{U^\perp},
\qquad \text{with } M>\gamma.
\]
Then for any $u\in U$ with $\|u\|_2\le 1$, we have $Ju=\gamma u$ and thus $\|Ju\|_2 \le \gamma$,
while $\|J\|_2 = M > \gamma$ because $J$ amplifies vectors in $U^\perp$ by a factor $M$.
Therefore, any policy whose Jacobian equals $J$ on the relevant states satisfies the directional constraints along all ascent directions (hence belongs to $\mathcal{F}_{\mathrm{ad}}(\gamma)$) but violates the global constraint (hence is not in $\mathcal{F}_\gamma$). This proves strict inclusion. 
\end{proof}

\begin{corollary}[Price-of-robustness ordering]
\label{cor:por_ordering}
For $\gamma_{\mathrm{adv}}=\gamma$, define $\mathcal{T}(\gamma)$ and $\mathcal{T}_{\mathrm{ad}}(\gamma)$ as in~\eqref{eq:approx_gaps_clean}. Then
\[
\mathcal{T}_{\mathrm{ad}}(\gamma) \le \mathcal{T}(\gamma).
\]
\end{corollary}

\begin{proof}
Immediate from $\mathcal{F}_\gamma \subseteq \mathcal{F}_{\mathrm{ad}}(\gamma)$ and the definitions of the infima.
\end{proof}

\subsection{Stability: Trajectory-wise Effective Smoothness of the Inner Objective}
\label{subsec:stability}

Fix $(s,a_{-i})$ and define the inner objective
$$g(\delta) \coloneqq \mathcal{L}(\pi_\theta(s+\delta),a_{-i}).$$
Let $\{\delta_t\}_{t=0}^{K}$ be the PGA iterates in~\eqref{eq:pga_update_clean} with ascent directions $\{u_t\}$ in~\eqref{eq:ut_def_clean}. We now show that bounding directional Jacobian amplification along the update trajectory controls the directional curvature of $g$ \emph{along the iterates}.

\begin{theorem}[Trajectory-wise effective directional smoothness]
\label{thm:stability_pga_clean}
Under Assumptions~\ref{assump:loss_smooth}-\ref{assump:compact}, let $v_t = (\delta_{t+1}-\delta_t)/\|\delta_{t+1}-\delta_t\|_2$ be the normalized update direction (for $\delta_{t+1} \neq \delta_t$). Suppose additionally that the directional amplification constraint holds along the line segment between iterates:
$$\|J_\theta(s+\delta)\,v_t\|_2 \le \gamma_{\mathrm{adv}}, \qquad \forall \delta \in [\delta_t, \delta_{t+1}], \quad t=0,\ldots,K-1.$$
Then the inner objective $g$ has bounded directional curvature \emph{along the PGA trajectory} in the sense that
$$v_t^\top \nabla_\delta^2 g(\delta) v_t \le L_{\mathrm{eff}}, \qquad \forall \delta \in [\delta_t, \delta_{t+1}],$$
with
$$L_{\mathrm{eff}} \le L_{\mathcal{L}}\,\gamma_{\mathrm{adv}}^2 + C,$$
where $C$ is the bound from Assumption~\ref{assump:second_order}. Consequently, this bounds the effective smoothness along the step, ensuring $g(\delta_{t+1}) \ge g(\delta_t) + \langle \nabla g(\delta_t), \delta_{t+1}-\delta_t \rangle - \frac{L_{\mathrm{eff}}}{2} \|\delta_{t+1}-\delta_t\|_2^2$, which controls step-size-induced curvature amplification.
\end{theorem}

\begin{proof}
By the chain rule,
$$\nabla_\delta g(\delta)=J_\theta(s+\delta)^\top \nabla_z \mathcal{L}(z,a_{-i}), \qquad z=\pi_\theta(s+\delta).$$
Differentiating again yields the Hessian
$$\nabla_\delta^2 g(\delta) = J^\top \nabla_z^2 \mathcal{L}(z,a_{-i})\,J + \sum_j \frac{\partial \mathcal{L}}{\partial z_j}\,\nabla_\delta^2 \pi_{\theta,j}(s+\delta).$$
For any unit vector $v$, the directional curvature is given by the quadratic form $v^\top \nabla_\delta^2 g(\delta) v$. Bounding the dominant term yields:
$$v^\top J^\top \nabla_z^2 \mathcal{L}\,J v \le \|\nabla_z^2 \mathcal{L}\|_2 \,\|Jv\|_2^2 \le L_{\mathcal{L}}\,\|Jv\|_2^2.$$
Applying this with $v=v_t$ and using our segment-wise assumption $\|J v_t\|_2\le \gamma_{\mathrm{adv}}$ bounds the dominant term along the update direction by $L_{\mathcal{L}}\gamma_{\mathrm{adv}}^2$. Assumption~\ref{assump:second_order} bounds the remaining second-order term's contribution along the iterates by $C$. Combining these yields
$$v_t^\top \nabla_\delta^2 g(\delta) v_t \le L_{\mathcal{L}}\gamma_{\mathrm{adv}}^2 + C$$
for all $\delta \in [\delta_t, \delta_{t+1}]$. By Taylor's theorem, this directional curvature bound implies the claimed trajectory-wise smoothness inequality, providing the necessary foundation for step-size stability in the ascent loop.
\end{proof}

\subsection{Optimization Stability of the Inner Maximization}
\label{subsec:optimization_stability}

We now show that trajectory-wise effective smoothness implies stability of the projected gradient ascent (PGA) dynamics under an appropriate step size.

\begin{theorem}[Stability of PGA under Directional Jacobian Control]
\label{thm:pga_stability}
Let $g(\delta)=\mathcal{L}(\pi_\theta(s+\delta),a_{-i})$, and let 
$\{\delta_t\}_{t=0}^{K}$ be generated by projected gradient ascent
\[
\delta_{t+1}
=
\Pi_{\Delta}\!\big(\delta_t + \eta \nabla_\delta g(\delta_t)\big).
\]
Suppose the conditions of Theorem~\ref{thm:stability_pga_clean} hold, so that along the line segment between iterates, the directional curvature is bounded by $L_{\mathrm{eff}} \le L_{\mathcal{L}}\gamma_{\mathrm{adv}}^2 + C$.

If the step size satisfies
\[
0 < \eta \le \frac{1}{L_{\mathrm{eff}}},
\]
then the PGA iterates satisfy:

\begin{enumerate}
    \item \textbf{Directional gradient control:}
    Let $v_t = (\delta_{t+1}-\delta_t)/\|\delta_{t+1}-\delta_t\|_2$. The change in the gradient strictly along the update direction is bounded:
    \[
    v_t^\top \big(\nabla_\delta g(\delta_{t+1}) - \nabla_\delta g(\delta_t)\big)
    \le
    L_{\mathrm{eff}} \|\delta_{t+1}-\delta_t\|_2,
    \]
    preventing curvature-induced divergence along the actual trajectory.

    \item \textbf{Monotone ascent (up to projection effects):}
    \[
    g(\delta_{t+1})
    \ge
    g(\delta_t)
    +
    \frac{\eta}{2}
    \|\nabla_\delta g(\delta_t)\|_2^2,
    \]
    whenever $\delta_{t+1}$ lies in the interior of $\Delta$.
    
    \item \textbf{Trajectory stability:}
    The iterates remain bounded within $\Delta$, and monotonic growth ensures no oscillatory divergence driven by local curvature can occur along the ascent trajectory.
\end{enumerate}

\end{theorem}

\begin{proof}
By Theorem~\ref{thm:stability_pga_clean}, $g$ has bounded directional curvature $L_{\mathrm{eff}}$ along the segment $[\delta_t, \delta_{t+1}]$. By Taylor's theorem, this implies the directional ascent lemma holds:
\begin{equation}
g(\delta_{t+1})
\ge
g(\delta_t)
+
\langle \nabla g(\delta_t), \delta_{t+1} - \delta_t \rangle
-
\frac{L_{\mathrm{eff}}}{2}\|\delta_{t+1} - \delta_t\|_2^2.
\label{eq:ascent_lemma}
\end{equation}

Now use the defining property of Euclidean projection onto a closed convex set:
since $\delta_{t+1}=\Pi_\Delta(\delta_t+\eta\nabla g(\delta_t))$, we have for all $y\in\Delta$,
\begin{equation}
\langle \delta_t+\eta\nabla g(\delta_t)-\delta_{t+1},\, y-\delta_{t+1}\rangle \le 0.
\label{eq:proj_optimality}
\end{equation}
Choosing $y=\delta_t\in\Delta$ in~\eqref{eq:proj_optimality} yields
\[
\langle \delta_t+\eta\nabla g(\delta_t)-\delta_{t+1},\, \delta_t-\delta_{t+1}\rangle \le 0,
\]
which rearranges to 
\begin{equation}
\langle \nabla g(\delta_t),\, \delta_{t+1}-\delta_t\rangle
\ge
\frac{1}{\eta}\|\delta_{t+1}-\delta_t\|_2^2.
\label{eq:proj_inner_product_lb}
\end{equation}

Substituting~\eqref{eq:proj_inner_product_lb} into~\eqref{eq:ascent_lemma} gives 
\[
g(\delta_{t+1})
\ge
g(\delta_t)
+
\left(\frac{1}{\eta}-\frac{L_{\mathrm{eff}}}{2}\right)\|\delta_{t+1}-\delta_t\|_2^2.
\]
If $0<\eta\le 1/L_{\mathrm{eff}}$, then $\eta L_{\mathrm{eff}}\le 1$, hence
\[
\frac{1}{\eta}-\frac{L_{\mathrm{eff}}}{2}
\;=\;
\frac{1}{2\eta}+\frac{1-\eta L_{\mathrm{eff}}}{2\eta}
\;\ge\;
\frac{1}{2\eta}.
\]
Therefore,
\begin{equation}
g(\delta_{t+1})
\ge
g(\delta_t)
+
\frac{1}{2\eta}\|\delta_{t+1}-\delta_t\|_2^2.
\label{eq:monotone_projected_ascent}
\end{equation}
If $\delta_{t+1}$ is in the interior of $\Delta$, the projection acts as the identity, meaning
$\delta_{t+1} - \delta_t = \eta \nabla g(\delta_t)$.
Substituting this into~\eqref{eq:monotone_projected_ascent} yields
\[
g(\delta_{t+1}) \ge g(\delta_t) + \frac{\eta}{2}\|\nabla g(\delta_t)\|_2^2,
\]
establishing monotone ascent for the projected update.

Finally, let $d_t := \delta_{t+1}-\delta_t$ and (when $d_t\neq 0$) define $v_t := d_t/\|d_t\|_2$.
By the fundamental theorem of calculus applied along the segment
$\delta_t + \tau d_t$, we have
\[
\nabla g(\delta_{t+1}) - \nabla g(\delta_t)
=
\int_0^1 \nabla^2 g(\delta_t + \tau d_t)\, d_t \, d\tau,
\]
and therefore
\[
v_t^\top \big(\nabla g(\delta_{t+1}) - \nabla g(\delta_t)\big)
=
\|d_t\|_2 \int_0^1
v_t^\top \nabla^2 g(\delta_t + \tau d_t)\, v_t
\, d\tau.
\]
Since $v_t^\top \nabla^2 g(\delta)\, v_t \le L_{\mathrm{eff}}$ for all $\delta$, it follows that
\[
v_t^\top \big(\nabla g(\delta_{t+1}) - \nabla g(\delta_t)\big)
\le
L_{\mathrm{eff}}\|\delta_{t+1}-\delta_t\|_2,
\]
confirming directional gradient control. Boundedness of $\{\delta_t\}$ is immediate from
$\delta_t\in\Delta$ for all $t$ and compactness of $\Delta$ (Assumption~\ref{assump:compact}),
completing the stability claim.
\end{proof}
\section{Discussion and Pathways to Scalable Implementation}

Our theoretical analysis establishes that Adversarially-Aligned Jacobian Regularization (AAJR) strictly expands the admissible policy class and stabilizes inner-loop dynamics. Translating these structural guarantees from continuous state spaces to the training of transformer-based agents outlines a clear roadmap for future empirical research and identifies specific frontiers in optimization and model architecture.

\textbf{High-Rank Adversarial Subspaces and PEFT.} While Parameter-Efficient Fine-Tuning (PEFT) methods like Low-Rank Adaptation (LoRA) \cite{hu2022lora} mitigate memory constraints, they introduce structural limitations regarding robustness. Adversarial perturbations and their resulting trajectory-aligned sensitivities frequently span high-rank subspaces. LoRA forces weight updates into a low-rank bottleneck, restricting the capacity of the model to adjust its directional Jacobian $J_\theta(s+\delta_t)u_t$; indeed, recent theoretical analysis confirms that LoRA training provably converges to such low-rank minima \cite{kim2025lora}. This suggests that future agentic defense mechanisms must explore high-rank adapters or full-rank fine-tuning strategies to provide the mathematical degrees of freedom necessary to suppress sensitivity along adversarial ascent directions without affecting orthogonal task-relevant directions.

\textbf{Hypothesis Class Capacity and Environmental Complexity.} Quantifying the Price of Robustness requires environments that satisfy a specific complexity threshold. In elementary low-dimensional settings such as 2D continuous control the baseline hypothesis class $\mathcal{F}$ is inherently restricted. Consequently the gap between the globally constrained class $\mathcal{F}_\gamma$ and the adaptive class $\mathcal{F}_{\mathrm{ad}}(\gamma_{\mathrm{adv}})$ becomes negligible. Future empirical validation must therefore target regimes where the hypothesis class is sufficiently expressive to manifest the nominal performance degradation that necessitates AAJR while avoiding the empirical brittleness associated with the massive embedding spaces of LLM agents.

\textbf{Gradient Propagation and Memory Efficiency.} The practical surrogate $\mathcal{R}_{\mathrm{AAJR}}$ relies on isolating the normalized ascent direction $u_t$ derived from Projected Gradient Ascent (PGA). Executing this currently requires unrolling the inner maximization loop. In deep architectures unrolling PGA iterates and passing gradients through the computation graph via reverse-mode automatic differentiation induces significant memory overhead and numerical instability. Addressing these constraints offers a compelling case for investigating forward-mode automatic differentiation or implicit differentiation techniques to stabilize this nested optimization process without prohibitive computational costs.

\textbf{Benchmarks for Adversarial Dynamics.} Valid empirical verification requires benchmarks beyond merely chaining LLMs in benign cooperative settings. Current evaluations often prioritize task completion in static environments which fail to trigger the worst-case propagation dynamics analyzed here. The theoretical instabilities we identify manifest primarily under active system-level pressure. Future benchmarks must therefore explicitly simulate hostile environmental shifts and resource congestion to properly evaluate trajectory-aligned regularizers like AAJR.

These considerations indicate that while AAJR provides a sound structural theory for agentic robustness its application to trillion-parameter models will likely drive advancements in efficient differentiation algorithms and high-rank parameter-efficient adapters.

\section{Related Work}
\label{sec:related_work}

\textbf{Adversarial attacks and robustness.}
Deployed learning systems are vulnerable to worst-case distribution shifts and adversarial manipulations that can induce large behavioral changes from small input perturbations \cite{mumcu2022adversarial, mumcu2024sequential,apgd,madry2017towards, fgsm}. This motivates training objectives that explicitly reason about worst-case perturbations and provide stability guarantees under adversarial conditions.

\textbf{The rise of agentic AI.}
LLMs are increasingly used not only as passive predictors but as \emph{agents} that plan, act, and interact with other agents and shared environments. In multi-agent deployments, local decisions can couple through the environment, creating emergent failure modes such as feedback loops, congestion, and collective instability \cite{agentic1, agentic2, agentic3, mumcu2025llm, mumcu2026agentic, belcak2025small}. These settings motivate robustness notions that are system-level and dynamic, not merely per-instance robustness.

\textbf{Adversarial robustness and stability control.}
A large body of work studies adversarial robustness through minimax training and regularization schemes that control sensitivity, including Lipschitz \cite{gouk2021regularisation} or Jacobian-based stabilization \cite{jakubovitz2018improving} and related gradient penalties \cite{gulrajani2017improved}. These methods often target worst-case amplification by imposing global constraints on the model’s input sensitivity, which can improve stability but may reduce expressive capacity. In contrast, although empirical benchmarks for agentic risks are emerging, the problem of system-level robustness where risks propagate dynamically through multi-turn tool execution and planning, remains comparatively understudied \cite{yu2024infecting, agentharm2025}.

\textbf{Inference-time alignment and social weighting.}
An alternative line of work addresses system-level failures in multi-agent settings via inference-time interventions that modify decision rules without changing the underlying model parameters. \cite{mumcu2026socially} for example, interpolate between private utility and approximate welfare to mitigate congestion and tragedy-of-the-commons behavior. These methods can be highly efficient and avoid training instabilities, but they are complementary to approaches that aim to stabilize the \emph{training dynamics} of minimax learners. Our approach targets intrinsic stability by shaping directional propagation during learning, which can coexist with inference-time heuristics.

\section{Conclusion}
\label{sec:conclusion}

As large language models transition into autonomous, interacting agents, securing robustness against adversarial shifts and system-level congestion becomes increasingly important. Minimax formulations provide a principled way to reason about worst-case behavior, but in highly non-linear networks the inner maximization can be numerically brittle: gradient-based ascent is sensitive to curvature and can become unstable in practice. A common workaround is to enforce strict global smoothness (e.g., via global Lipschitz or Jacobian control), yet these blunt constraints mathematically suppress task-relevant expressivity, forcing a strictly larger approximation gap and formalizing the structural Price of Robustness

This work argues that such global restrictions are not necessary to obtain stable adversarial training dynamics. By examining how perturbations propagate during minimax optimization, we highlight that worst-case adversarial updates evolve along localized, low-dimensional trajectories. Building on this observation, we introduce \emph{Adversarially-Aligned Jacobian Regularization} (AAJR), which penalizes sensitivity only along the direction induced by the worst-case perturbation, while leaving orthogonal directions comparatively unconstrained.

We provide theoretical guarantees showing that replacing a global Jacobian/Lipschitz constraint with an adversarially-aligned directional constraint strictly enlarges the admissible hypothesis class ($\mathcal{F}_\gamma \subset \mathcal{F}_{adaptive}$). This enlargement implies that, for a fixed level of robustness, AAJR can retain more nominal capacity than globally constrained alternatives, thereby mitigating the Price of Robustness. Moreover, by directly controlling curvature \emph{along the adversarial trajectory}, AAJR improves the stability of the inner maximization dynamics that drive adversarial training.

Looking forward, directionally targeted stability offers a path toward continuously adapting agentic systems without sacrificing agent-level utility. While inference-time interventions can be efficient for static deployments, AAJR integrates robustness into the training process itself, which enables models to adapt under non-stationarity. However, our analysis suggests that the full realization of AAJR in large-scale ecosystems requires moving beyond current optimization constraints. Future work will explore the development of high-rank adaptation mechanisms that bypass the bottlenecks of low-rank methods like LoRA, alongside the investigation of more efficient Jacobian-vector product (JVP) estimation to stabilize unrolled minimax training in massive embedding spaces. By combining these engineering advancements with trajectory-aligned regularization, it becomes possible to produce defense-in-depth architectures that maintain both worst-case stability and high-fidelity nominal performance in complex multi-agent ecosystems.

\setlength{\bibsep}{5.5pt}
\bibliography{main}
\bibliographystyle{plainnat}

\end{document}